\documentclass[10pt,twocolumn,letterpaper]{article}
\usepackage{wrapfig}
\usepackage{cvpr}              

\usepackage{bbm}
\usepackage{amsmath,amsfonts} 
\usepackage[linesnumbered,ruled,vlined]{algorithm2e} 
\usepackage{graphicx}
\usepackage{textcomp}
\usepackage{amssymb}
\usepackage{amsthm}
\usepackage{xcolor}
\usepackage{multirow}
\usepackage{array}
\usepackage{makecell}
\usepackage{subcaption} 
\usepackage{placeins}
\usepackage[font=small, skip = 4pt]{caption}
\newtheorem{lemma}{Lemma}
\definecolor{cvprblue}{rgb}{0.21,0.49,0.74}
\usepackage[pagebackref,breaklinks,colorlinks,allcolors=cvprblue]{hyperref}

\def\paperID{*****} 
\def\confName{CVPR}
\def\confYear{2026}

\title{DANCE: Dynamic 3D CNN Pruning: Joint Frame, Channel, and Feature Adaptation for Energy Efficiency on the Edge}

\author{Mohamed Mejri \quad Ashiqur Rasul \quad Abhijit Chatterjee \\
School of Electrical and Computer Engineering, 
Georgia Tech, 
Atlanta, GA, USA\\
{\tt\small mmejri3@gatech.edu, arasul6@gatech.edu, abhijit.chatterjee@ece.gatech.edu}
}
\begin{document}

\maketitle

\begin{abstract}
Modern convolutional neural networks (CNNs) are workhorses for video and image processing, but fail to adapt to the computational complexity of input samples in a dynamic manner to minimize energy consumption. In this research, we propose \textit{DANCE}, a fine-grained, input-aware, dynamic pruning framework for 3D CNNs to maximize power efficiency with negligible to zero impact on performance. In the proposed two-step approach, the first step is called \textit{activation variability amplification (AVA)} and the 3D CNN model is retrained to increase the variance of the magnitude of neuron activations across the network in this step, facilitating pruning decisions across diverse CNN input scenarios.
In the second step, called \textit{adaptive activation pruning (AAP)}, a lightweight \textit{activation controller network} is trained to \textit{dynamically prune frames, channels and features} of 3D convolutional layers of the network (different for each layer), based on statistics of the outputs of the first layer of the network. 
Our method achieves substantial savings in multiply–accumulate (MAC) operations and memory accesses by introducing sparsity within convolutional layers. Hardware validation on the NVIDIA Jetson Nano GPU and the Qualcomm Snapdragon 8 Gen 1 platform demonstrates respective speedups of 1.37$\times$ and 2.22$\times$, achieving up to 1.47$\times$ higher energy efficiency compared to the state of the art.
 
\footnote{Code and experimental data of this research work will be published after acceptance of the article.}
\end{abstract}

\section{Introduction}
Convolutional neural networks (CNNs) are widely used for computer vision applications involving both 3D and 2D structures, such as for videos and images, respectively, for tasks such as gesture classification, action recognition, video caption generation, object detection and tracking, semantic segmentation and so on. 
However, convolutional neural networks are computationally expensive, with costs that grow with network depth and scaling from 2D (image \cite{russakovsky2015imagenet,tan2019efficientnet}) to 3D (video  or point cloud \cite{tran2015learning, point_cloud}) processing. To address this, structured pruning has emerged as a promising solution ~\cite{filters2016pruning,luo2017thinet,liu2017learning,he2017channel}, reducing both memory footprint and computational requirements. Pruning can also be applied dynamically, removing parameters adaptively depending on the input ~\cite{jha_Dyn,mathurmind,kaushik_CDL,skipnet,chen2019you,gao2018dynamic,hua2019channel}. Despite the fact that dynamic pruning has been applied mostly to 2D CNNs, where focus is largely on coarse strategies such as channel pruning or layer removal techniques, which, while effective, remain limited in granularity and overall scope of power efficiency.

In this work, we propose \textit{DANCE}, a two-step,  fine-grained, input-aware pruning framework for 3D convolutional neural networks that dynamically removes unnecessary frames, channels, and features inside intermediate 4D ($height \times width \times channel \times frame$) activation tensors of the hidden layers of the network. The first step, referred to as \textit{activation variability amplification (AVA)}, consists of training the network while increasing the variance of activation magnitudes across the frames, channels, and features inside the network. This variance amplification facilitates later pruning by reinforcing critical neurons and suppressing redundant neurons. The second step, \textit{adaptive activation pruning (AAP)}, introduces a lightweight \textit{activation controller network} that dynamically prunes frames, channels, and features corresponding to different hidden layers of the network for each input sample, based on thresholds applied to the aggregated outputs (representing statistics) of the first layer of the network.
This enables the 3D CNN model to skip unnecessary neurons dynamically, allowing each input sample to trigger a tailored pruning pattern across the frames, channels, and features of each hidden layer.  Computational savings arise from skipping operations tied to pruned activations,
greatly reducing MAC operations and memory accesses. Additionally, removing noisy or irrelevant activations can enhance accuracy by filtering out noisy or misleading features. \textit{The sparsity of channels and feature activations varies both from frame to frame and across different convolutional layers in a 3D CNN}. The key contributions of this paper are:

\noindent (1) We propose a novel \emph{input-aware adaptive fine-grained pruning framework} for 3D convolutional neural networks, capable of dynamically pruning frames, channels, and features in 4D activation tensors of the hidden layers of the network, on a per-layer basis, based on statistics of the outputs of the first layer of the network.

\noindent (2) A \emph{two-step adaptation process is used}: (1) \textbf{Activation variability amplification (AVA)}, which increases activation variance across the 4D activation tensor dimensions, without degrading network performance, 
and (2) \textbf{Adaptive activation pruning (AAP)}, which employs a lightweight network to predict input-aware pruning thresholds for frame, channel, and feature pruning, allowing each input sample to trigger a tailored pruning pattern.

\noindent (3)The proposed method introduces structured sparsity to reduce compute and memory overhead while maintaining (and sometimes improving) accuracy compared to prior work. We demonstrate the versatility of our approach on an NVIDIA Jetson Nano GPU and a Snapdragon 8 Gen 1 (Samsung S22) mobile CPU. By leveraging custom Neon SIMD kernels, we have achieved computational speedups of \textbf{1.37$\times$} and \textbf{2.22$\times$}, respectively, with up to \textbf{1.47$\times$} higher energy efficiency over optimized baselines.

\vspace{-4pt}
\section{Related Works}
To develop input adaptive neural networks with dynamic computational graphs, researchers have approached the problem from three angles: sample-wise adaptive, spatially adaptive, and temporally adaptive. Sample-wise adaptive models adapt their computational graphs according to the complexity of individual input instances, whereas spatially adaptive and temporally adaptive neural networks modify and optimize themselves along spatial and temporal dimensions. Layer and channel skipping\cite{jha_Dyn, skipnet, kaushik_CDL} is a popular technique to avoid unnecessary computation in neural networks, and this method involves adopting gating or halting mechanisms using an auxiliary policy network. BlockDrop\cite{blockdrop} suggested a policy gradient reinforcement learning (REINFORCE)\cite{reinforce, reinforce_1, reinforce_2} algorithm for block selection within residual neural network architectures. The authors of Conditional Deep Learning \cite{kaushik_CDL} propose to add linear classification heads at the end of each convolution block, creating a cascaded architecture for early exit neural networks. Class-aware channel pruning \cite{crisp, classawarepruning} methods dynamically reduce the model computational burden by selecting the important kernels based on class-aware saliency scores. Furthermore, a unified static and dynamic pruning framework has been suggested in the works of \cite{bilevel_pruning} with differentiable gating and channel selection mechanism using the Gumbel \cite{gumbel_1, gumbel_2} sigmoid trick. Aforementioned optimization techniques can be extended to 3D convolutional neural networks for efficient processing of their relatively demanding computational task as well. Few research works have been conducted on static compression and pruning of 3D neural networks \cite{RT3D, ranp}, however, most of the investigations lack the essence of adaptation in their performance optimization scheme. Li et al \cite{2DNot2D} formulated an adaptive policy to optimize the use of 3D convolutional kernels where a reinforcement learning based selection network governs the usage of frame and convolutional filters. Additionally, Chen et al \cite{freq_compress} proposed a pruning technique on the basis of converting convolutional kernels in the temporal axis to frequency domain representations, reducing the computational complexity by removing insignificant frequency components.
\vspace{-4pt}

\section{Overview}
\label{over}
The framework for the proposed dynamic 3D CNN pruning approach is given in 
Figure \ref{overview_high}. Figure \ref{overview_high} (top) shows input frames (video) to our 
system that are processed by the first convolutional layer (3D Conv 1) of the network. The outputs of the first layer  are passed to an \textit{Adaptive Activation Pruning (AAP) Controller} module. The AAP controller generates control signals for the active pruning of frames, channels, and features of all subsequent 3D convolutional layers in the network, as described below.

\begin{figure}[!ht]
    \centering
    \includegraphics[width=\linewidth]{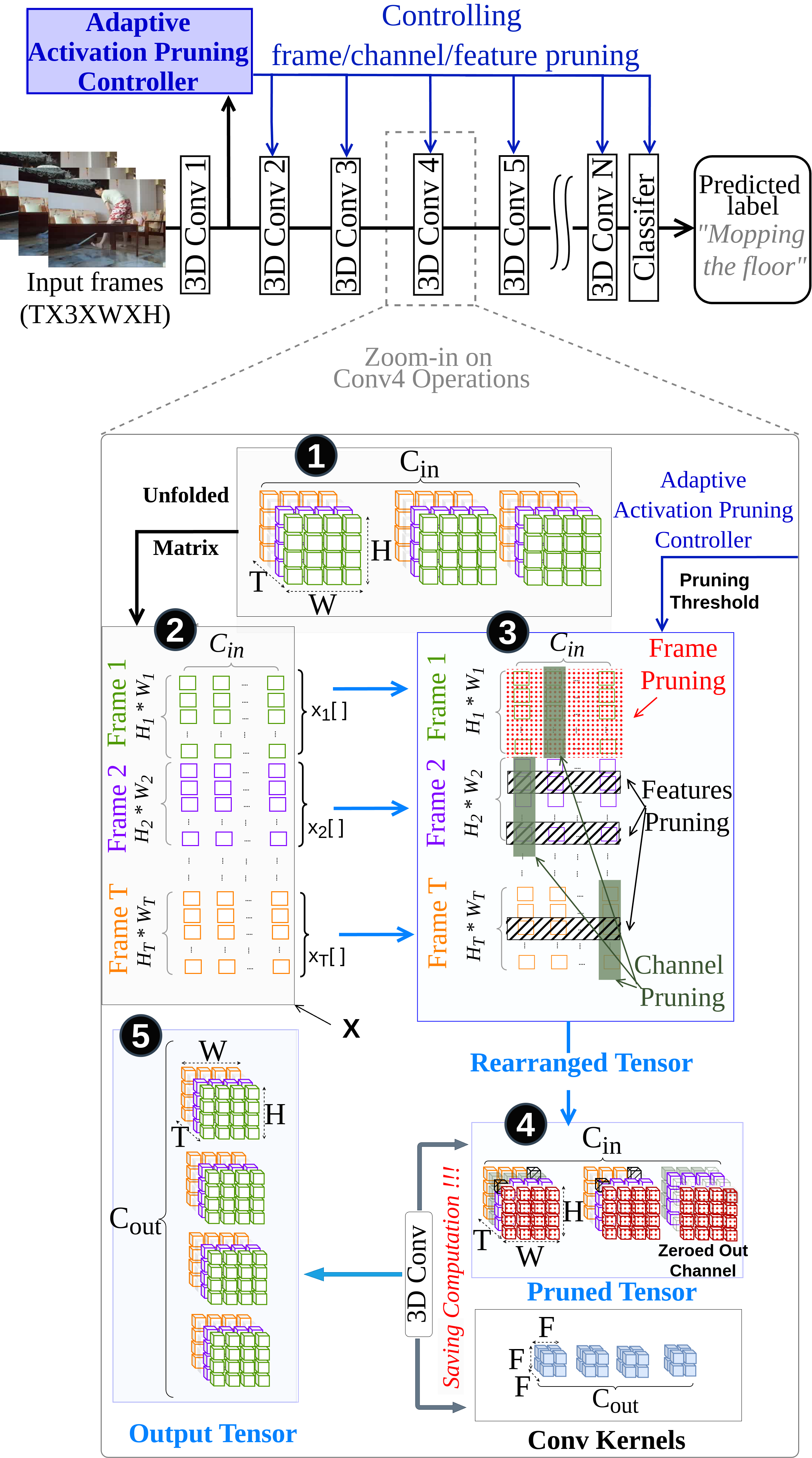}
    \caption{Illustration of 4D tensor pruning in intermediate convolution layers. Sparsity in the activation tensor is introduced at three levels by dynamic pruning: for frames (dotted in red), channels (shaded in green), and features (lined in black and white).}
    \label{overview_high}
\end{figure}

An expanded view of the ``3D Conv 4'' layer is shown in Figure \ref{overview_high}, serving as the reference structure for all subsequent layers throughout the network. The inputs to this layer consist of $C_{in}$ channels, with each $H \times W$ feature map corresponding to the respective channel filter, replicated over \textit{T} temporal frames of video input, forming a 4D tensor as illustrated in box 1. Box (2) of Figure \ref{overview_high} shows the unfolded ``flattened'' representation (matrix) of the 4D tensor above. The rows and columns of this matrix are modulated by the AAP controller as shown in box (3) of Figure \ref{overview_high}. Frame pruning eliminates one or more complete frames of the matrix (${H_i} \times {W_i}$ rows, $1 \leq i \leq T$, box (3) top). 
Channel pruning eliminates complete columns of the matrix, while feature pruning eliminates corresponding weights of feature maps across all the channels of the network corresponding to different time frames (different for different time frames). The rearranged tensor is shown in box (4) of Figure \ref{overview_high}.
The pruned tensor is subsequently convolved with the kernels of the current layer, 3D Conv 4, to produce the output, shown as the 4D tensor in box (5), which serves as the input to the next layer in the network.
Computational savings are achieved during convolution by skipping the operations associated with pruned frames, channels, and features. Based on the framework described in Figure \ref{overview_high} the dynamic pruning approach proceeds in two steps.

\begin{figure*}[!ht]
    \centering
    \includegraphics[width=0.9\textwidth]{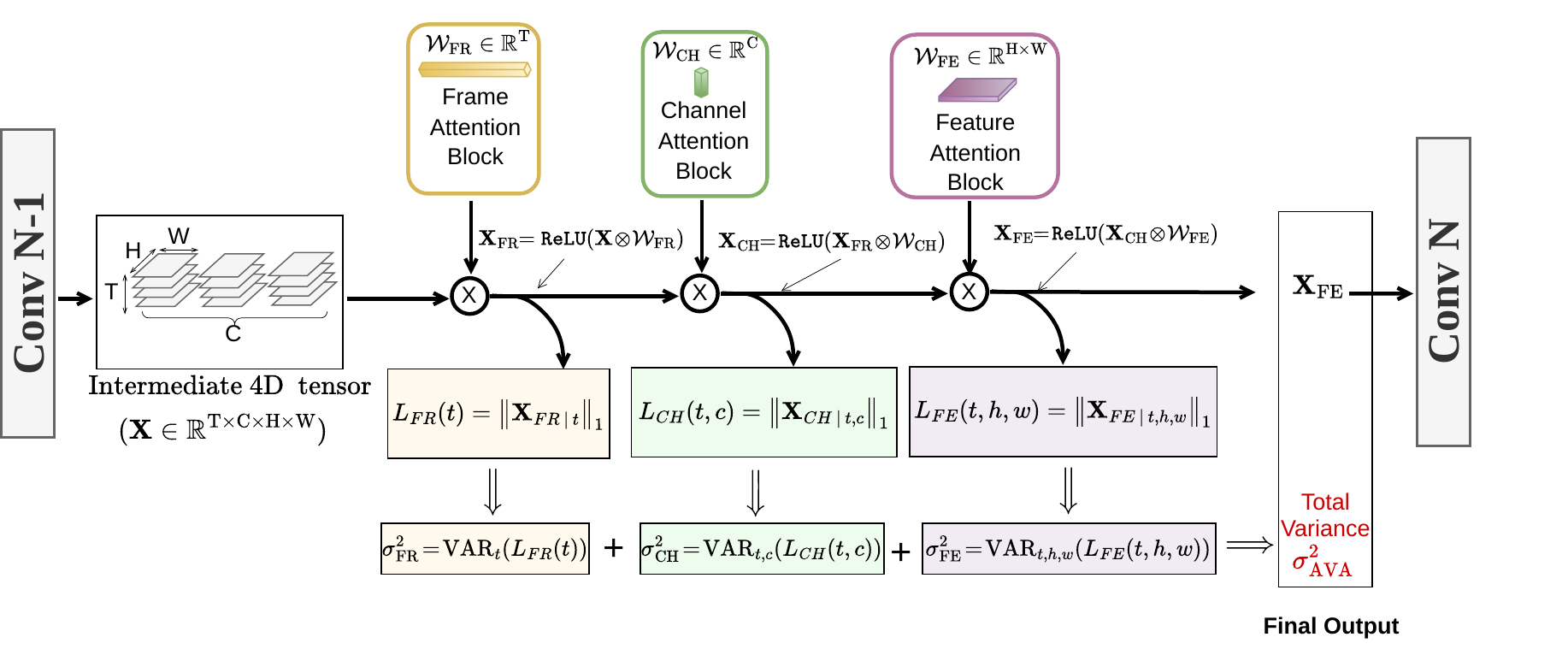}
    \caption{Overview of the Activation Variability Amplication (AVA) mechanism. Variance along the frame ($L_{FR}$), channel ($L_{CH}$) and feature ($L_{FE}$) dimensions are aggregated to determine the total variance $\sigma^2_{AVA}$, which is used as parameter in the loss function to train the model for boosting variance within activations.}
    \label{fig:AVA_overview}
\end{figure*}

\noindent \textit{Step (1)}: To facilitate dynamic pruning it is necessary to amplify the effects of diverse network inputs (video frames across time, images) on the outputs of all the layers of the 3D convolution network of Figure \ref{overview_high}. A novel 3D CNN training method called \textit{activation variability amplification (AVA)} is used to increase the variance of the distributions of the magnitudes of neuron activations across frames, channels, and features of the layers of the network of Figure \ref{overview_high}. After model training, the post-AVA magnitude distribution exhibits higher variation across all corresponding  4D tensor dimensions. This induced variation enables clearer separation between relevant and irrelevant frames, channels, and features for more effective dynamic pruning. 

\noindent \textit{Step (2)}:  The outputs of the first layer of the network are passed to an adaptive activation pruning (AAP) controller module. This eliminates
frames, channels, and features of individual layers of the network based on processing of the outputs of the first layer and a generated threshold for each pruned element (inspired by the component retention criterion in Principal Component Analysis (PCA)\cite{PCA_1}, adapted here to operate on activation magnitudes as described in Step 1). It is worth mentioning that, only the parameters involved in the adaptive activation pruning controller are learned in the second step (refer to figure \ref{overview_high}), while other parameters are kept frozen.
In the following, Steps (1) and (2) are discussed in Sections \ref{sec:AVA} and \ref{sec:AAP}, respectively.
\vspace{-6pt}

\section{Activation Variability Amplification (AVA)}
\label{sec:AVA}
Figure~\ref{fig:AVA_overview} describes the proposed \textbf{adaptive variability amplification (AVA)} training procedure of Step 1 discussed in Section \ref{over}. Inspired by the Convolutional Block Attention Module (CBAM)~\cite{10.1007/978-3-030-01234-2_1}, which enhances the representational power of 2D convolutional neural networks through channel and features attention, the AVA module extends this concept by applying three hierarchical attention modules: frame, channel, and feature based attention. Unlike CBAM, our approach serves a different purpose: to quantify and adaptively re-weight neuron activation variability across the frame, channel, and feature dimensions. These attention modules help reweighting frames, channels and features within the activation tensor which helps increasing their magnitude variability.

\noindent\textit{Frame Attention Module: }
Consider the ``flattened'' 4D tensor $X$ shown in Box (2) of Figure \ref{overview_high}. This is also interchangeably represented as a matrix $X$[], with sub-matrices [$x_1$[], $x_2$[], .... $x_T[$]. Each sub-matrix
$x_t$[j,k] is of dimension $H_t$*$W_t$ by $C_{in}$  (rows by columns) corresponding to each frame $t$, $1 \leq t \leq T$, as shown in the figure. Consider a trainable vector $W_{FR}$ of dimension $T$, ${W_{FR}}$ = [${w_{FR}}(1)$, ${w_{FR}}(2)$,... ${w_{FR}}(T)]$. If we assume that $H_t$ is equal to a fixed height $H$ and $W_t$ is equal to a fixed width $W$ for all $1 \leq t \leq T$, we compute: 
\begin{equation}
L_{\text{FR}}(t) = \frac{\left|{w_{FR}}(t)\right|}{C \cdot H \cdot W}
\sum_{j=1}^{H*W} \sum_{k=1}^{C} \left| {x_t}[j,k] \right|
\label{eq:one}
\end{equation}
The variance of $L_{\text{FR}}$(t) over all $1 \leq t \leq T$ is denoted as
$\sigma_{\text{FR}}^{2}$. Maximizing $\sigma_{\text{FR}}^{2}$ during training helps in determining the relevance of different frames relative to others on a per-input basis. The sub-matrices $x_t$[],
$1 \leq\ t \leq T$ are weighted by the scalars ${w_{FR}} (t)$ (represented in tensor form as $X_{FR} = \texttt{ReLU}(X \times {W_{FR}})$ in Figure \ref{fig:AVA_overview}) and passed on to the channel attention module.  

\noindent\textit{Channel Attention Module: } 
The weighted matrices   ${x'}_t$[j,k] = $\texttt{ReLU}({w_{FR}}(t)$ * $x_t$[j,k]), $1 \leq t \leq T$, 
are forwarded to a channel attention module, where channel-wise modulation is performed using the trainable vector ${W_{CH}} = [{w_{CH}}(1), {w_{CH}}(2),... {w_{CH}}(C_{in})]$. A frame-dependent, channel-wise quantity 
$\mathbf{L}_{\text{CH}}$(t,c), $1 \leq t \leq T$, $1 \leq c \leq {C_{in}}$, is computed as:
\begin{equation}
L_{\text{CH}}(t,c) = \frac{\left|{w_{CH}}(c)\right|}{H \cdot W}
\sum_{j=1}^{H*W} \left| {{x'}_t}[j,c] \right|
\label{eq:two}
\end{equation}
The variance of  $L_{\text{CH}}(t,c)$ over all $1 \leq t \leq T$, $1 \leq c \leq {C_{in}}$, is computed as $\sigma_{\text{CH}}^{2}$. Maximizing $\sigma_{\text{CH}}^{2}$ during training helps in determining the relevance of different channels relative to others on a per-input basis.
The tensor $X_{CH}$ depicted by the columns of the sub-matrices ${x'}_t$[] (representing channels), weighted by the scalars ${w_{CH}} (c)$, $1 \leq t \leq T$, $1 \leq c \leq {C_{in}}$, given as
$X_{CH} = \texttt{ReLU}({X_{FR}} \times {W_{CH}})$  in Figure \ref{fig:AVA_overview} are passed on to the feature attention module. 

\noindent\textit{Feature Attention Module: } 
The columns of the sub-matrices ${x'}_t$[j,k] above, are weighted by the quantities $w_{CH}(k)$,  $1 \leq k \leq {C_{in}}$, to yield the sub-matrices ${x^*}_t$[j,k]. These are forwarded to a feature attention module where
feature-wise modulation is performed using the trainable parameter $\mathcal{W}_{\text{FE}} = [{w_{FE}}(1), {w_{FE}}(2),... {w_{FE}}(H*W)]$. 
A frame-dependent feature-based quantity $\mathbf{L}_{\text{FE}}$ is derived as follows:
\begin{equation}
L_{\text{FE}}(t,f) = \frac{\left|{w_{FE}}(f)\right|}{C}
\sum_{c=1}^{C}
\left| {{x^*}_t}[f,c] \right|,
\label{eq:three}
\end{equation}
where $f$ is a row index for the sub-matrix ${x^*}_t$[], $1 \leq f \leq {H*W}$.
The variance of  $L_{\text{FE}}(t,f)$ over all $1 \leq t \leq T$, $1 \leq f \leq {H*W}$, is computed as $\sigma_{\text{FE}}^{2}$. Maximizing $\sigma_{\text{FE}}^{2}$ during training helps in determining the relevance of different features relative to others on a per-input basis.
The tensor $X_{FE}$ depicted by the rows of the sub-matrices ${x^*}_t$[] (representing features), weighted by the scalars ${w_{FE}} (f)$, $1 \leq t \leq T$, $1 \leq f \leq {H*W}$, given as
$X_{FE} = \texttt{ReLU}({X_{CH}} \times {W_{FE}})$  in Figure \ref{fig:AVA_overview} are passed on to further convolutional layers of the network and for computation of the objective function for network training.  

The 
frame, channel and feature based variances referred as $\sigma^2_{\text{FR}}$, $\sigma^2_{\text{CH}}$ and $\sigma^2_{\text{FE}}$ respectively are aggregated to generate the total variance denoted $\sigma^{2}_{\text{AVA}}$. The total variances from all the AVA modules inside the 3D CNN are summed together to generate the final variance referred to as $\sigma^{2}_{f}$.  The 3D CNN is then trained to jointly minimize the cross-entropy loss $\mathcal{L}_{CE}$ between predicted and ground-truth labels and maximize the aggregated variance $\sigma^2_{\text{AVA}}$. The overall objective function is defined as:
\[
\mathcal{L}_{f} = \mathcal{L}_{CE} - \beta \, {\sigma^2_{f}}
\]    
where $\beta$ is a small positive weighting factor that regulates the contribution of the variance term, ensuring that the optimization primarily focuses on reducing classification error while increasing feature variability.

The choice of the standard deviation also serves the purpose of sparsity maximization. 
In our setting, the quantities $L_{\text{CH}}(t,c)$, $L_{\text{FE}}(t,f)$, and 
$L_{\text{FR}}(t)$ are first normalized with respect to the $\ell_1$ norm, 
i.e., their entries are nonnegative and rescaled to sum to one. 
Under this normalization, Lemma~\ref{lemma_1} shows that maximizing the standard deviation is equivalent 
to maximizing sparsity in the sense of the Hoyer measure~\cite{yang2019deephoyer}, 
defined for a vector $\mathbf{x} \in \mathbb{R}^D$ as
\[
H(\mathbf{x})=
\frac{\sqrt{D} - \dfrac{\|\mathbf{x}\|_1}{\|\mathbf{x}\|_2}}
{\sqrt{D} - 1}.
\]
\begin{lemma}\label{lemma_1}
On the probability simplex $\Delta^{D-1}$, the standard deviation $\operatorname{std}(\mathbf{x})$ and the Hoyer measure $H(\mathbf{x})$ are both strictly increasing functions of $\|\mathbf{x}\|_2$. Consequently, $\arg\max_{\mathbf{x} \in \Delta^{D-1}} \operatorname{std}(\mathbf{x}) = \arg\max_{\mathbf{x} \in \Delta^{D-1}} H(\mathbf{x})$.
\end{lemma}

\begin{proof}
Since $\|\mathbf{x}\|_1 = 1$, then $\bar{x} = 1/D$ and $\operatorname{std}(\mathbf{x})^2 = \frac{1}{D}\|\mathbf{x}\|_2^2 - \frac{1}{D^2}$, which is strictly increasing in $\|\mathbf{x}\|_2$. Similarly, for $\mathbf{x} \in \Delta^{D-1}$, the Hoyer measure simplifies to $H(\mathbf{x}) = \frac{\sqrt{D} - 1/\|\mathbf{x}\|_2}{\sqrt{D}-1}$. Because the term $-1/\|\mathbf{x}\|_2$ is strictly increasing in $\|\mathbf{x}\|_2$, $H(\mathbf{x})$ is also a strictly increasing function of $\|\mathbf{x}\|_2$. Since both objectives are monotonic transformations of the same norm, they share identical maximizers.
\end{proof}

\begin{figure}[!ht]
    \centering
    \includegraphics[width=\linewidth]{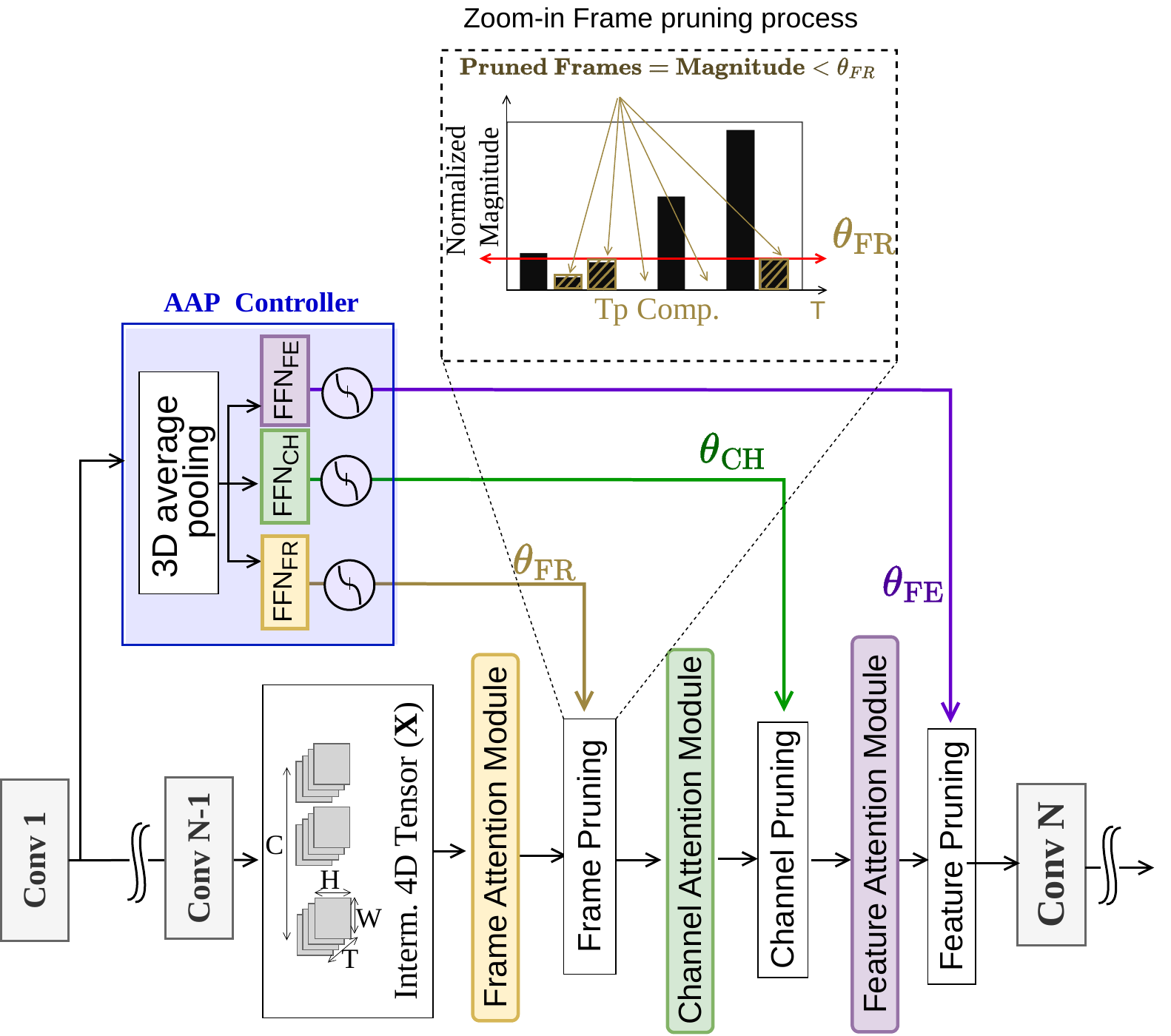}
    \caption{Adaptive Activation Pruning (AAP): The AAP function is applied sequentially on frames, channels, and features (depicted from top to bottom) to produce a structured sparsity pattern in the activation tensor fed into the subsequent convolutional layer.}
    \label{fig:AAP}
\end{figure}

\section{Adaptive Activation Pruning (AAP)}
\label{sec:AAP}

Adaptive activation pruning (AAP) introduces sparsity into the activation tensors of intermediate convolutional layers of the network based on precomputed thresholds. Following the variability amplification phase, high variance within activation tensor dimensions stimulates large magnitude differences among frames, channels, and features, which in turn facilitates elimination of unnecessary frames, channels, and features (those with low magnitude) from the computation graph and following the step, the \textit{weights of the 3D convolutional network are frozen and only the parameters of the AAP module are trained}.
Figure~\ref{fig:AAP} illustrates the processing involved in the Adaptive Activation Pruning (AAP) module, and proceeds in three consecutive steps. 

The output of the $1^{st}$ convolutional layer of the 3D CNN serves as input to the AAP controller, which utilizes 3D average pooling followed by three FFNs ($\text{FFN}_{\text{FR}}$, $\text{FFN}_{\text{CH}}$, $\text{FFN}_{\text{FE}}$) and sigmoid activations to produce global pruning thresholds $\theta_{\text{FR}}$, $\theta_{\text{CH}}$, and $\theta_{\text{FE}}$. These thresholds are shared across all layers and operate within $[0, 1]$ due to the normalization of the importance metrics ($L$). While a global threshold may be less efficient than layer-wise adaptation, it avoids the linear overhead of per-layer modules, preserving the net computational benefit. Pruning is executed as follows: frames $t$ are suppressed if $L_{FR}(t) \leq \theta_{\text{FR}}$ (Eq. \ref{eq:one}), channels if $L_{CH}(t,c) \leq \theta_{\text{CH}}$ (Eq. \ref{eq:two}), and features if $L_{FE}(t,c) \leq \theta_{\text{FE}}$ (Eq. \ref{eq:three}).

As illustrated in Figure~\ref{fig:AAP}, frame, channel, and feature pruning are applied after each corresponding attention module to the intermediate 4D tensor ($\mathbf{X}$) using the retention thresholds described above.


Since the model has already been trained with the AVA module, retraining the entire network while pruning is unnecessary and could even reduce variability in intermediate activations. Therefore, the 3D CNN weights are kept frozen, and only the AAP controller network is trained to produce higher retention thresholds while maintaining classification performance.

However, training the AAP controller network jointly with adaptive activation pruning poses a challenge, as pruning involves non-differentiable binary mask generation. To address the non-differentiability issue, we adopt the Straight-Through Estimator (STE) technique~\cite{bengio2013estimating}, which applies a hard (non-differentiable) operation in the forward pass and a soft counterpart in the backward pass, allowing gradients to flow through the 3D CNN. For the binary mask generation, we employ a soft gating operation inspired by the Gumbel-Softmax technique \cite{gumbel_1, gumbel_2, gumbelsoftmaxselectivenetworks}, as shown in Equation~\ref{eq:gate}:
\begin{equation}\label{eq:gate}
y =
\begin{cases}
\mathbbm{1}\!\left[\dfrac{x - \theta + n}{\tau} \ge 0\right], 
& \text{(forward pass)}, \\[8pt]
\sigma\!\left(\dfrac{x - \theta + n}{\tau}\right), 
& \text{(backward pass)}.
\end{cases}
\end{equation}
\begin{figure}[ht]
    \centering
    \includegraphics[width=0.7\columnwidth]{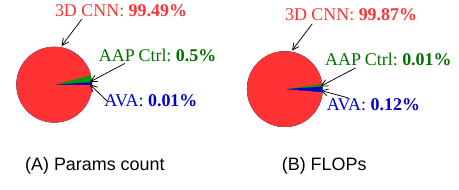}
    \caption{Overhead of the 3D CNN, AVA modules, and AAP controller}
    \label{fig:repartition_overhead}
\end{figure}
\begin{figure*}[!ht]
\centering
\begin{subfigure}{0.3\textwidth}
    \centering
    \includegraphics[width=\linewidth]{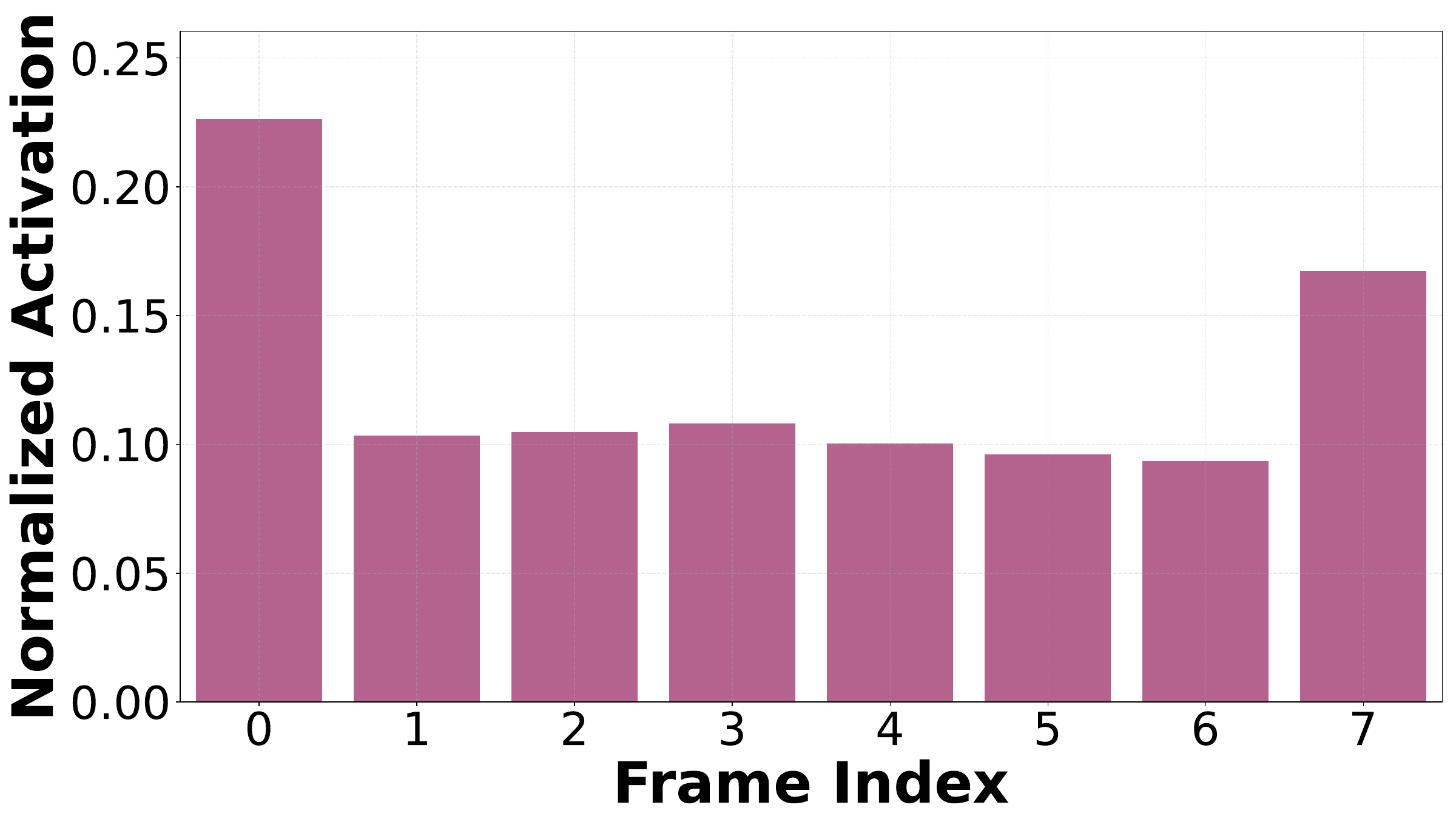}
    \caption{}
\end{subfigure}%
\hfill
\begin{subfigure}{0.3\textwidth}
    \centering
    \includegraphics[width=\linewidth]{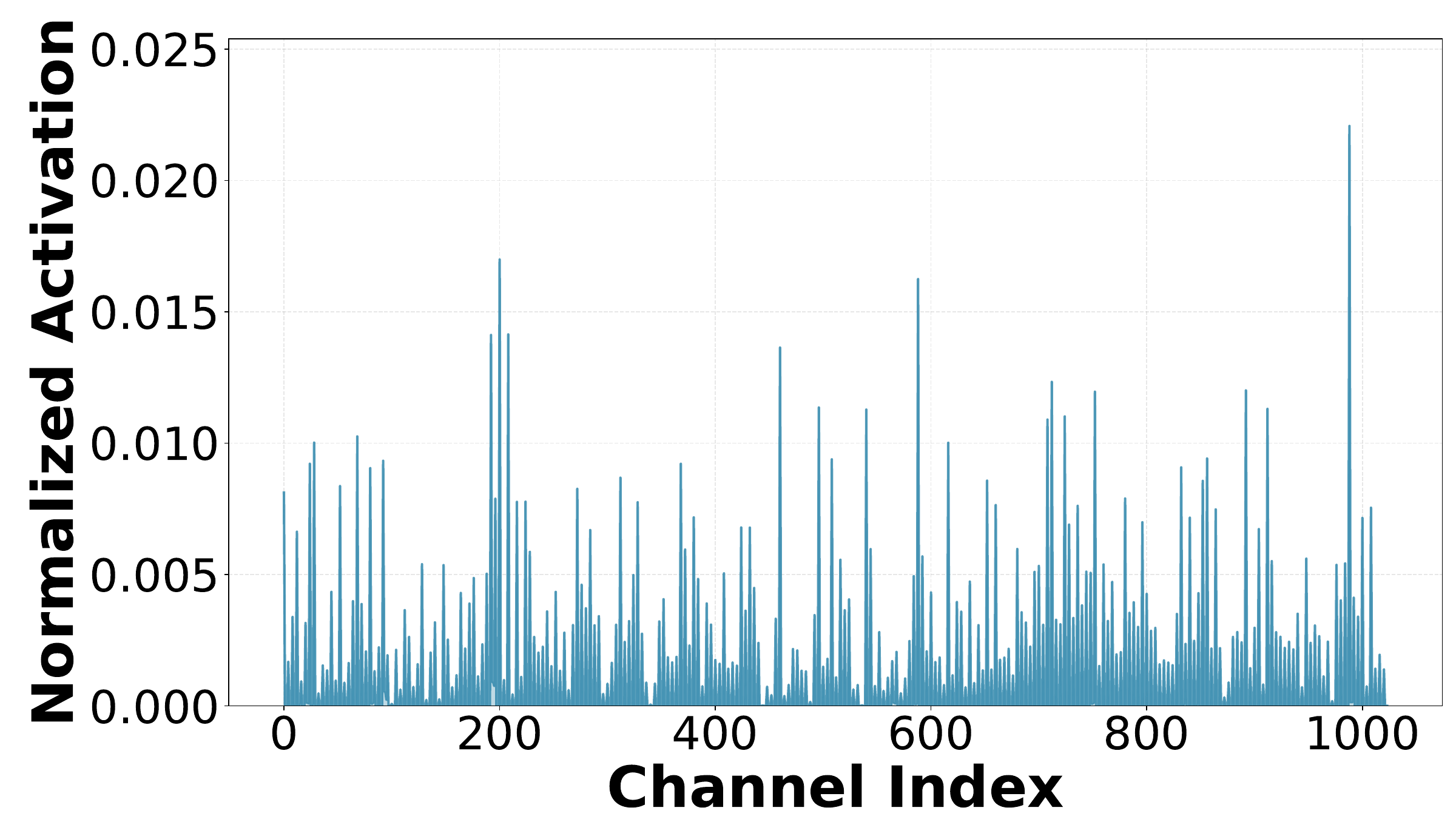}
    \caption{}
\end{subfigure}%
\hfill
\begin{subfigure}{0.3\textwidth}
    \centering
    \includegraphics[width=\linewidth]{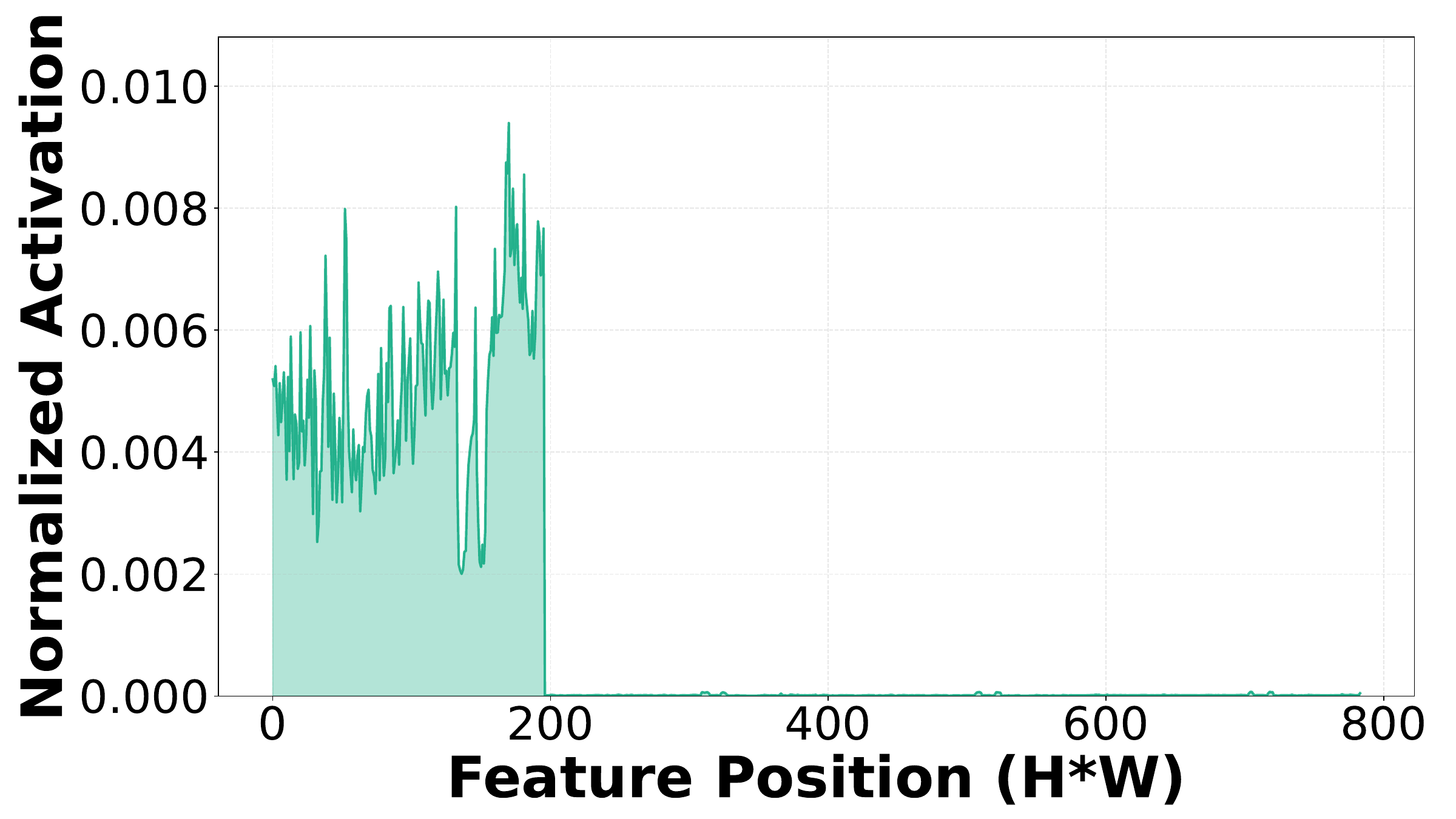}
    \caption{}
\end{subfigure}

\vspace{-2pt} 

\begin{subfigure}{0.3\textwidth}
    \centering
    \includegraphics[width=\linewidth]{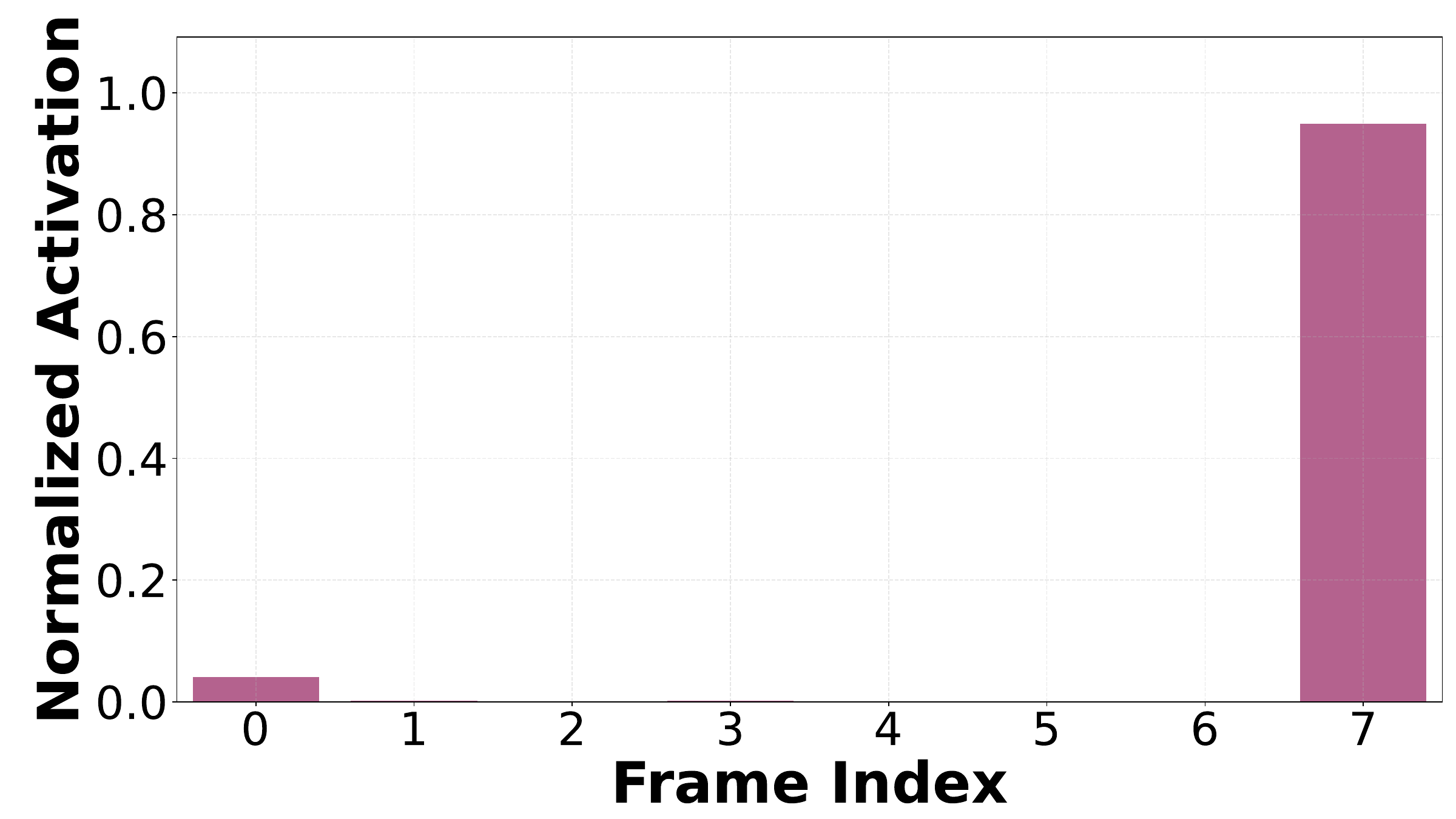}
    \caption{}
\end{subfigure}%
\hfill
\begin{subfigure}{0.3\textwidth}
    \centering
    \includegraphics[width=\linewidth]{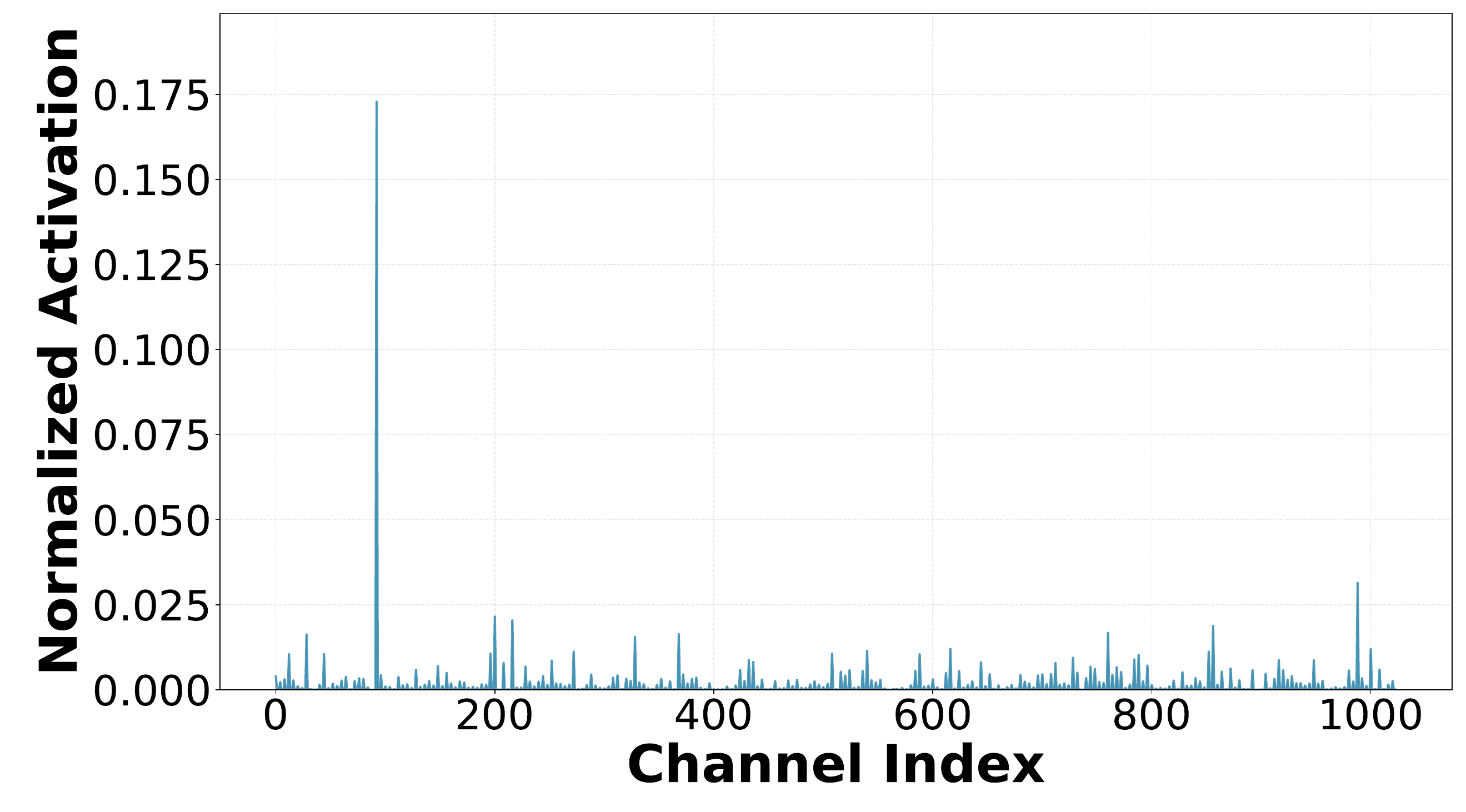}
    \caption{}
\end{subfigure}%
\hfill
\begin{subfigure}{0.3\textwidth}
    \centering
    \includegraphics[width=\linewidth]{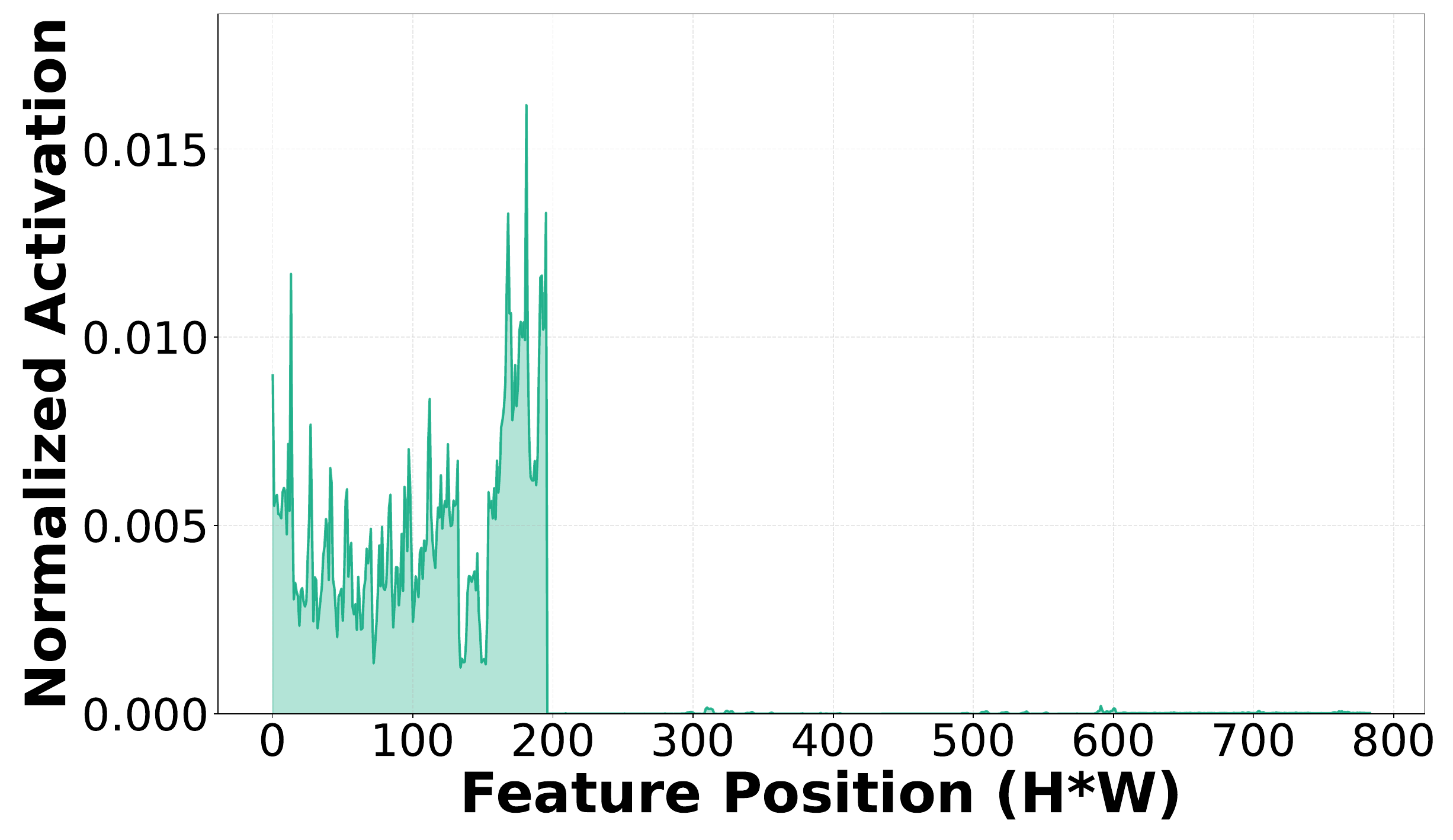}
    \caption{}
\end{subfigure}

\caption{Visualization of frame, channel, and feature magnitude distribution before and after AVA}
\label{fig:twobythree}
\end{figure*}
In Equation~\ref{eq:gate}, $\sigma$ represents the sigmoid function, $x$ denotes the normalized magnitude of frames, channels, or pixels; $\theta$ represents the retention threshold, and $n \sim \mathcal{N}(0, \epsilon)$ is a Gaussian noise term with zero mean and small variance $\epsilon$, introduced to smooth the threshold boundary. This stochasticity enables more stable gradient flow near decision boundaries. The output $y$ populates the binary mask with 1 if the normalized magnitude is above the threshold and zero otherwise.    
The AAP controller network is trained to adaptively select the highest possible retention thresholds, resulting in higher pruning rates while maintaining accuracy. The final objective loss is thus defined as:
\[
\mathcal{L}_f = \mathcal{L}_{CE} - \lambda \cdot \sum_{\theta \in \text{AAP Modules}} (\theta_{\text{FR}} + \theta_{\text{CH}} + \theta_{\text{FE}})
\]
where $\lambda$ is a small hyperparameter controlling the regularization strength.

\section{Experimental Results}
We first present simulation results for computational efficiency and final model accuracy after pruning on a baseline dataset, and compare the same against state-of-the-art methods. Next, we conduct an ablation study to assess (1) the importance of the AVA step and (2) the effect of retraining the full network together with the AAP controller on accuracy and pruning ratio. We then report the overhead introduced by the AAP controller and AVA modules relative to the 3D CNN model. Subsequently, we describe the pruning ratios for individual components; frames, channels, and features, and provide visual profiles of their magnitude distributions before and after the AVA procedure. We also include per-layer FLOP profiles before and after pruning. Finally, we show the speedup and energy efficiency achieved on validation hardware.

\begin{table}[!ht]
\centering
\footnotesize
\resizebox{\columnwidth}{!}{
\begin{tabular}{c c c c c c}
\toprule
Model & \makecell{Sparsity \\ Scheme} & \makecell{Pruning \\ Rate} & \makecell{Baseline \\ Acc. (\%)} & \makecell{Top-1 \\ Acc. (\%)} & $\Delta$ Acc. (\%)\\
\midrule
\multirow{5}{*}{R(2+1)D\cite{R(2+1)D}}
& Filter      & 2.6$\times$   & \multirow{4}{*}{94.5} & 90.5   & -4.00 \\
& Vanilla     & 2.6$\times$   &                        & 91.7   & -2.80 \\
& KGS\cite{RT3D}         & 3.2$\times$   &                        & 92.0   & -2.50 \\
& KGRC        & 3.02$\times$  &                        & 92.65  & -1.85 \\
\cmidrule(lr){2-6}
& \textit{Our method}
              & \textbf{4.0$\times$} 
              & 91.5 & \textbf{93.05} & \textbf{+1.55} \\
\midrule
\multirow{9}{*}{C3D\cite{C3D}}
& DCP\cite{DCP}         & 1.99$\times$  & \multirow{5}{*}{82.77} & 71.72 & -11.05 \\
& FP\cite{FP}          & 2.02$\times$  &                        & 75.44 & -7.33 \\
& TP\cite{TP}          & 2.02$\times$  &                        & 70.00 & -12.77 \\
& MDP\cite{MDP}         & 2.00$\times$  &                        & 76.82 & -5.95 \\
& MDP+KD      & 2.00$\times$  &                        & 80.10 & -2.67 \\
\cmidrule(lr){2-6}
& KGR          & 1.76$\times$  & \multirow{3}{*}{82.82} & 81.84 & -0.98 \\
& KGC          & 1.93$\times$  &                        & 81.39 & -1.43 \\
& KGRC
              & 3.04$\times$
              &                & 80.21 & -2.61 \\
\cmidrule(lr){2-6}
& \textit{Our method}
              & \textbf{3.49}$\times$ 
              & 87.9 & \textbf{87.93} & \textbf{+0.03} \\
\bottomrule
\end{tabular}
}
\caption{Comparison of pruning methods on UCF101.}
\label{tab:pruning_comparison}
\vspace{-2pt}
\end{table}

\begin{table}[t]
\centering
\footnotesize
\setlength{\tabcolsep}{6pt} 
\begin{tabular}{c c c c}
\toprule
Sparsity Scheme & \makecell{Pruning \\ Rate } & \makecell{Accuracy \\ (\%)} & $\Delta$ Acc. (\%) \\
\midrule
DCP\cite{DCP}      & 2.02$\times$ & 40.59 & -7.6 \\
FP\cite{FP}       & 2.02$\times$ & 40.98 & -7.2\\
TP\cite{TP}       & 2.11$\times$ & 34.84 & -13.3\\
MDP\cite{MDP}      & 2.14$\times$ & 43.20 & -4.9\\
MDP+KD   & 2.14$\times$ & 45.62 & -2.6\\
        \textit{Our method} & \textbf{2.77$\times$} & \textbf{46.17} & \textbf{-2.02} \\
\bottomrule
\end{tabular}
\caption{Performance comparison of model compression approaches on HMDB51 (Split 1), using pruning ratio and accuracy improvements over the C3D baseline.}
\label{tab:hmdb51_pruning_comparison}
\vspace{-10pt}
\end{table}

Tables \ref{tab:pruning_comparison} and \ref{tab:hmdb51_pruning_comparison} compare respectively, our method against state-of-the-art (SOTA) pruning schemes on UCF101~\cite{UCF101} and HMDB51~\cite{HMDB} using the R(2+1)D~\cite{R(2+1)D} and C3D~\cite{C3D} models pretrained on Kinetics~\cite{kinetics}, reporting pruning rates, baseline accuracies, and post-pruning top-1 accuracies. Across all settings, our dynamic, fine-grained pruning consistently outperforms existing static approaches. On UCF101 with R(2+1)D~\cite{R(2+1)D}, we achieve a 4× pruning rate compared to the best SOTA technique, (which achieves 3.02×), while improving accuracy by 1.5\%, whereas all static baselines incur accuracy degradation. For UCF101 with C3D~\cite{C3D}, our method provides a 3.49× pruning rate—higher than the best SOTA at 3.04×, with no accuracy drop. On HMDB51 with C3D, we again surpass all prior work with a 2.77× pruning rate versus 2.14×, and only a minimal 2\% accuracy reduction. Overall, the tables show that dynamic pruning coupled with fine-grained sparsification yields more optimal compression and superior accuracy preservation than the coarse-grained static strategies used in current state-of-the-art methods.
The higher baseline accuracies (91.5\% for $R(2+1)D$ and 87.9\% for $C3D$) result from the AVA module acting as a feature enhancer that strengthens representations by suppressing redundant activations. This dual-purpose design facilitates fine-grained pruning while simultaneously boosting model robustness, enabling superior compression rates without the accuracy degradation typical of SOTA methods.


\begin{table}[ht]
\centering
\footnotesize
\setlength{\tabcolsep}{8pt}
\resizebox{\linewidth}{!}{
\begin{tabular}{c c c c c c c}
\toprule
& \multicolumn{3}{c}{UCF101} & \multicolumn{3}{c}{HMDB51} \\
\cmidrule(lr){2-4} \cmidrule(lr){5-7}
Metric & w/ AVA & w/o AVA & \makecell{AAP \\ Co-train} & w/ AVA & w/o AVA & \makecell{AAP \\ Co-train} \\
\midrule
\makecell{Pruning rate} & 3.49 & 1.1 & 2.86 & 2.77 & 1.2 & 2.32 \\
Accuracy     & 87.93 & 85.04 & 86.7 & 46.17 & 44.5 & 45.57 \\
\bottomrule
\end{tabular}
}
\caption{Ablation study on pruning rate and accuracy for UCF101 and HMDB51 with and without the AVA step.}
\label{tab:ava_comparison}
\vspace{-10pt}
\end{table}

\begin{table}[ht]
\centering
\footnotesize
\setlength{\tabcolsep}{8pt} 

\resizebox{0.8\linewidth}{!}{%
\begin{tabular}{c c c c c}
\toprule
Dataset & \makecell{Frame \\ only} & \makecell{Channel \\ only} & \makecell{Feature \\ only} & Combined \\
\midrule
UCF101  & 1.03$\times$ & 2.47$\times$ & 1.50$\times$ & 3.49$\times$ \\
HMDB51  & 1.35$\times$ & 2.62$\times$ & 1.39$\times$ & 2.77$\times$ \\
\bottomrule
\end{tabular}%
}

\caption{\small Pruning-ratio distributions (frame, channel, feature, combined) for C3D on UCF101 and HMDB51.}
\label{tab:pruning_ratios}
\vspace{-10pt}
\end{table}
Table \ref{tab:ava_comparison} shows that AVA is essential; without it, uniform activation magnitudes prevent the AAP controller from establishing effective pruning thresholds. Furthermore, co-training the CNN and AAP controller disrupts the structured variability required for pruning, decreasing accuracy without providing additional overhead reduction.

Figure \ref{fig:repartition_overhead} further shows that both AVA and the AAP controller add negligible memory and computation relative to the main CNN, confirming that these components are necessary and introduce minimal overhead.  

Table \ref{tab:pruning_ratios} shows how frame, channel, and feature pruning individually contribute to pruning the full C3D model on UCF101 and HMDB51 datasets. In both cases, channels are the sparsest component and provide over twice the overhead reduction of frames and features. Features exceed frames only in activation pruning, as removing a full frame would significantly degrade CNN accuracy.  

As depicted in figure \ref{fig:twobythree}, variance is boosted along the frame, channel, and feature dimensions of the C3D model by suppressing the redundant neurons, and the sub-figures show the contrast in normalized activation magnitude at the fourth convolutional layer of the C3D model for a sample input. From our experiments, it has been observed that the adaptive variability amplification framework is more effective along the frame or channel dimension, compared to the feature. Training the model to increase variance within the activation tensor, therefore, facilitates the elimination of redundant neurons, which in turn, promotes sparsity when passed through the Adaptive Activation Pruning module. 

\begin{figure}[!h]
    \centering
    \includegraphics[width=0.8\linewidth]{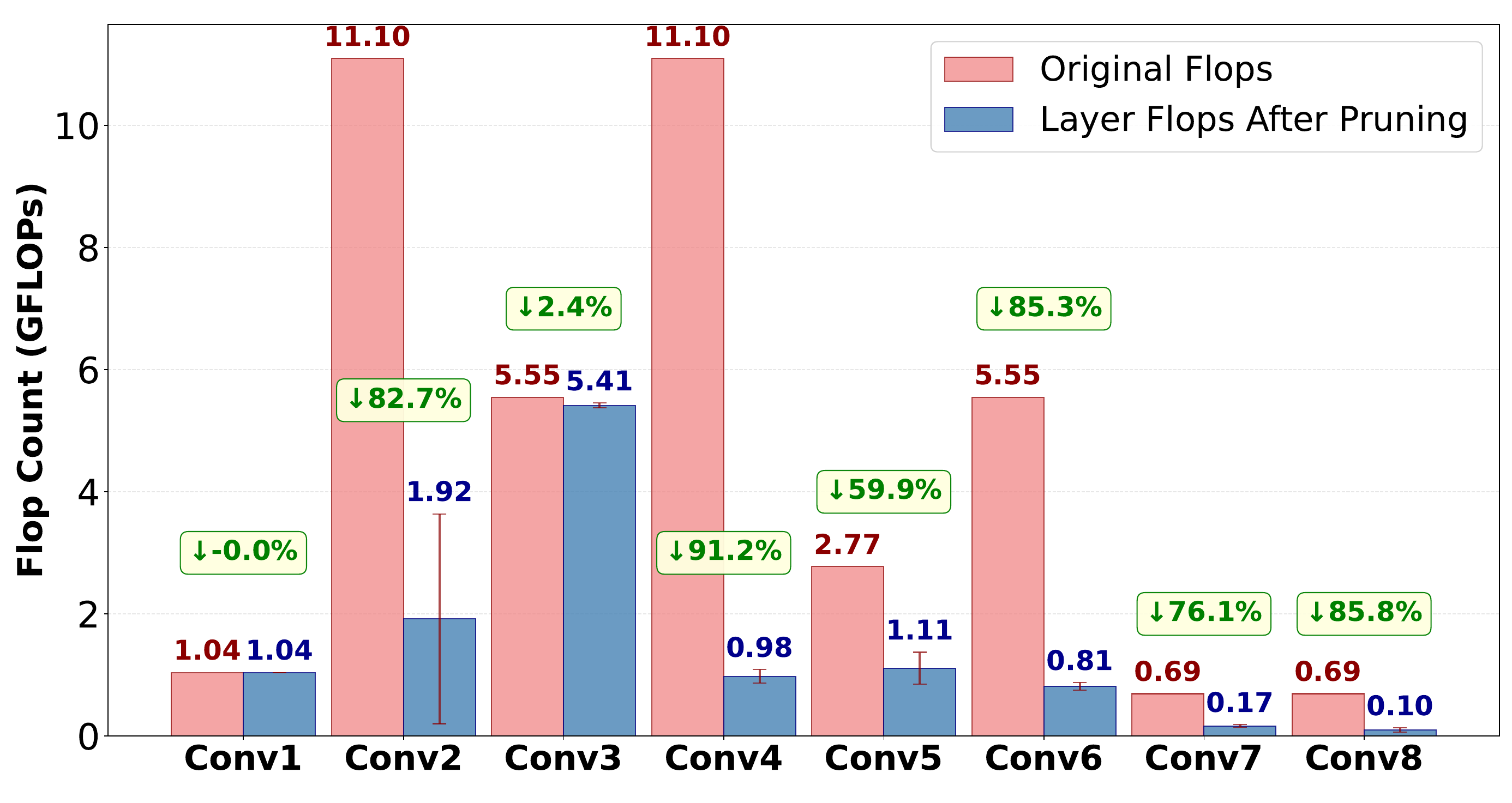}
    \caption{FLOPs reduction per conv. layer in the C3D model.}
    \label{fig:flop_comparison}
\end{figure}
Figure \ref{fig:flop_comparison} shows the average FLOP reduction across the C3D convolutional layers. Red bars indicate dense-layer FLOPs, while blue bars show FLOPs after introducing activation sparsity. The first layer shows no reduction because its output is needed by the AAP controller. From the second layer onward, sparsity has a clear impact. The red vertical lines denote each layer’s FLOP range, with Conv2 showing the greatest variation. Conv2 and Conv4 also achieve the largest pruning gains due to their high original computational overhead.
\begin{figure}
    \centering
    \begin{subfigure}{0.4\linewidth}
        \centering
        \includegraphics[width=\linewidth]{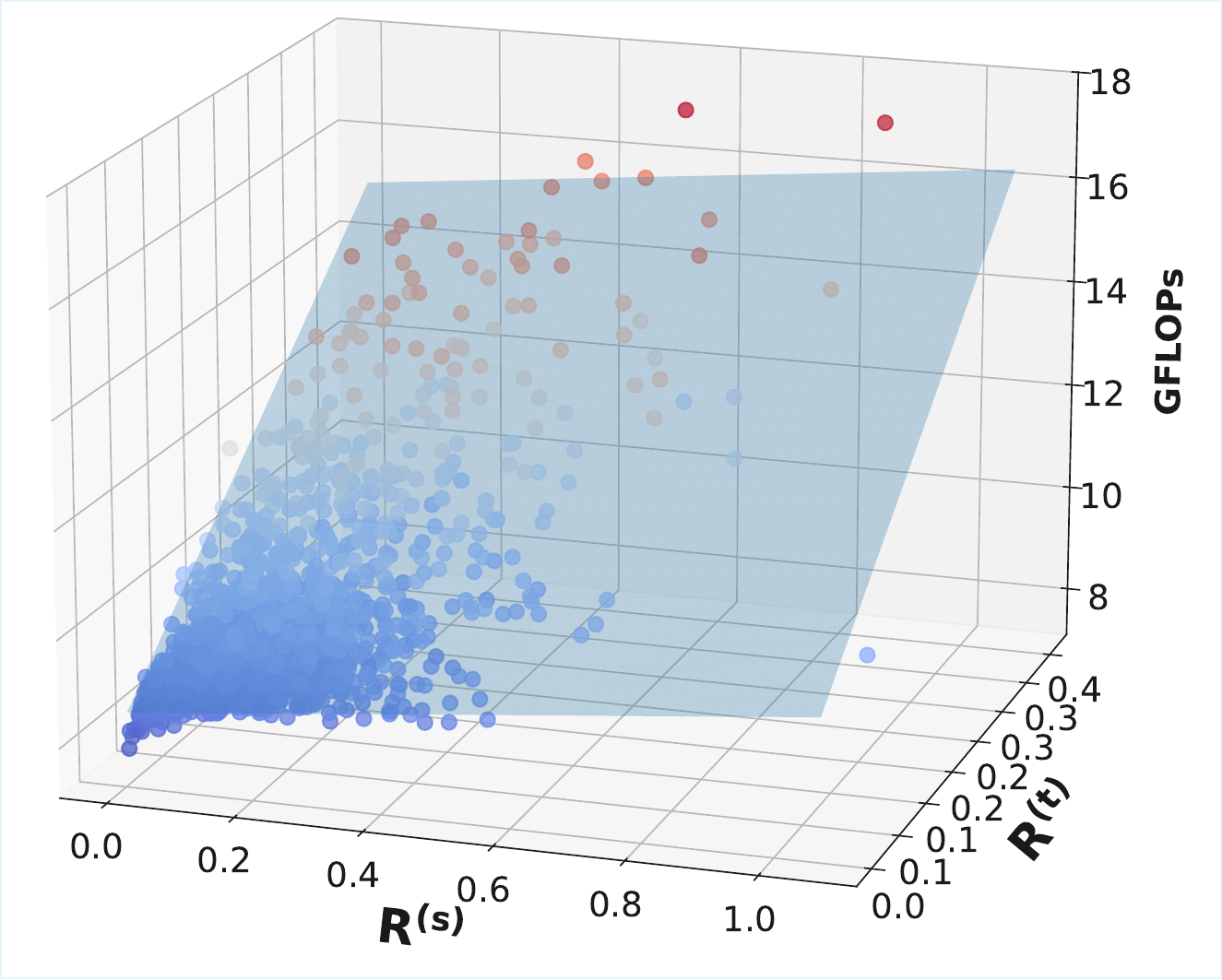}
        \caption{C3D ($R^2 = 0.745$)}
        \label{fig:hist1}
    \end{subfigure}
    \hfill
    \begin{subfigure}{0.4\linewidth}
        \centering
        \includegraphics[width=\linewidth]{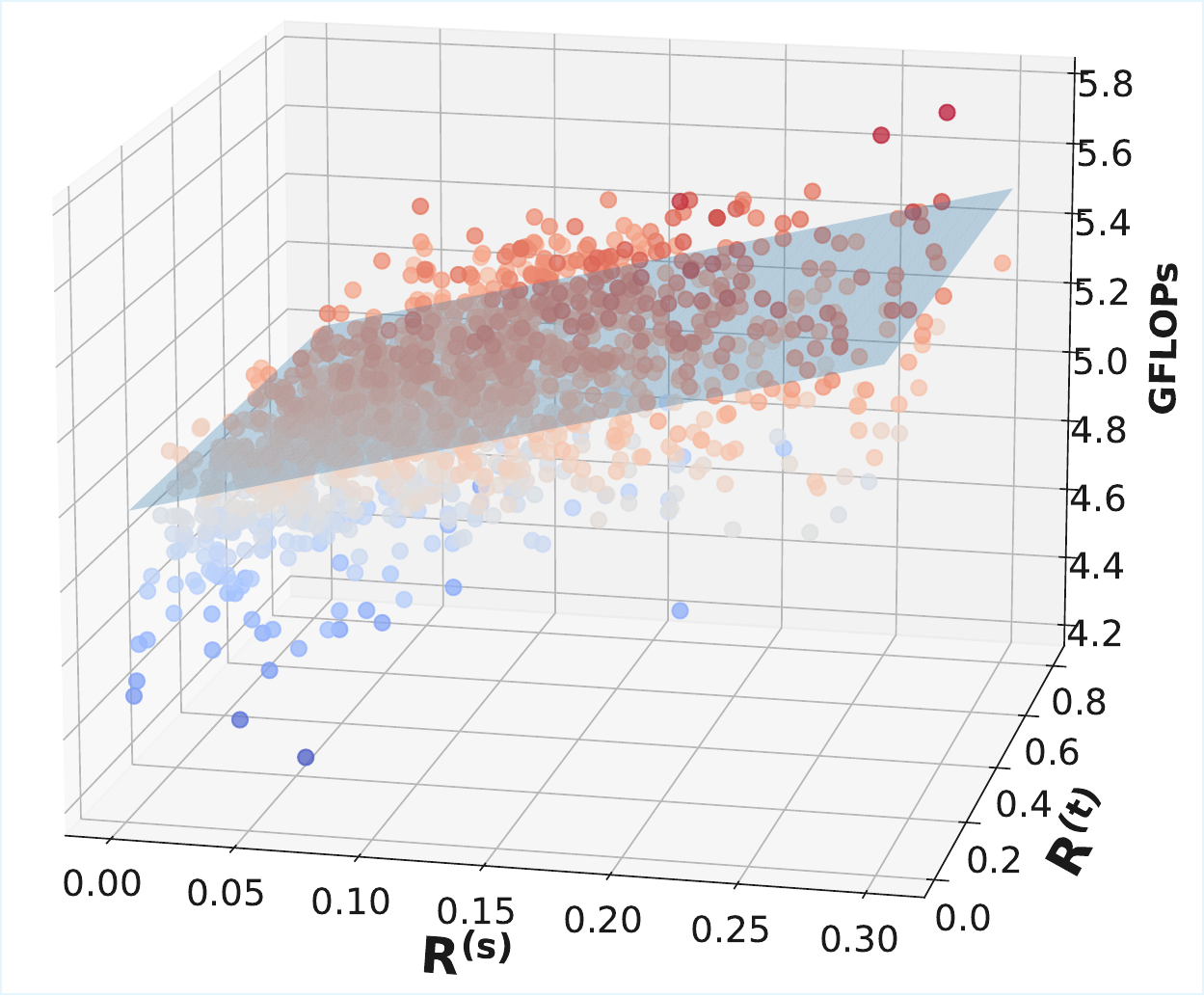}
        \caption{R2+1D ($R^2=0.271$)}
        \label{fig:hist2}
    \end{subfigure}
    \caption{Inference GFLOPs distribution on UCF-101 for different adaptively pruned models w.r.t. spatial and temporal information}
    \label{fig:Distribution}
    \vspace{-0.1in}
\end{figure}
Figure~\ref{fig:Distribution} illustrates the distribution of inference computational overhead (FLOPs) after adaptive pruning of the two baseline 3D CNN architectures: C3D (Figure~\ref{fig:hist1}) and R2+1D (Figure~\ref{fig:hist2}). The computational cost is analyzed with respect to the input information level of each video.

The spatial and temporal information measures, denoted as $\mathbf{R^{s}}$ and $\mathbf{R^{t}}$, quantify the structural complexity within frames and the variability across frames, respectively. Inspired by the information-theoretic formulation in~\cite{shen2006robust} and later used in~\cite{herrmann2020learning}, both metrics are derived from spatio-temporal DCT energy ratios between high- and low-frequency components and are log-normalized as $R = \log\!\left(1 + \frac{E_{\text{high}}}{E_{\text{low}}}\right)$ to stabilize variance and reduce skewness.

\begin{figure*}[ht]
    \centering
    \begin{subfigure}[t]{0.45\linewidth}
        \centering
        \includegraphics[width=\linewidth]{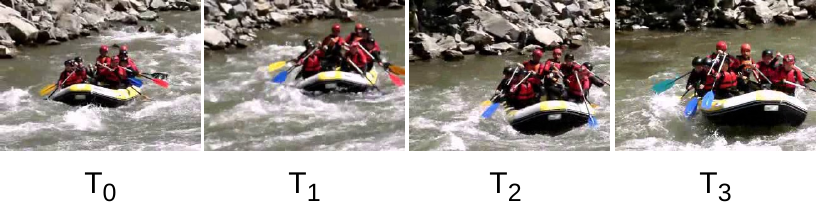}
        \caption{Highest-complexity example: Rafting activity}
        \label{fig:juggling}
    \end{subfigure}
    \hfill
    \begin{subfigure}[t]{0.45\linewidth}
        \centering
        \includegraphics[width=\linewidth]{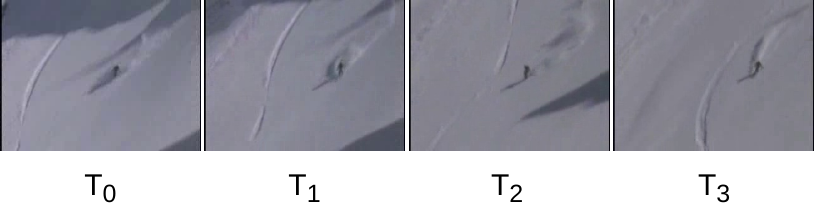}
        \caption{Lowest-complexity example: Skiing activity}
        \label{fig:skiing}
    \end{subfigure}
    
    \caption{Two extrema examples from UCF101 processed with C3D. The left video frames (Rafting) correspond to the maximum computational cost (14.5 GFLOPs), while the right video frames (Skiing) correspond to the minimum computational cost (7.8 GFLOPs).}
    \label{fig:extrema_frames}
\end{figure*}

We model the relationship between intrinsic video information content and the computational effort required during inference. A strong correlation is observed between $\left[ R^{s}, R^{t} \right]$ and FLOPs. For C3D, the multiple correlation coefficient and coefficient of determination are $R = 0.863$ and $R^2 = 0.745$, respectively, indicating that 74.5\% of the variance is explained by input complexity. In contrast, R2+1D exhibits weaker correlation ($R = 0.521$, $R^2 = 0.271$).

These results confirm that the proposed adaptive framework is input-dependent, allocating more computation to videos that are more complex from an information-theoretic perspective. The disparity between C3D and R2+1D stems from architectural differences: R2+1D is a residual network in which skip connections limit the propagation of sparsity across layers, whereas C3D follows a VGG-style architecture without skip connections, allowing sparsity to accumulate and making its cost more sensitive to input complexity.

Figure~\ref{fig:extrema_frames} presents representative frames from the most and least computationally intensive video samples processed by the C3D model. Specifically, Figures~\ref{fig:juggling} and~\ref{fig:skiing} display four downsampled frames from the rafting and skiing activities, respectively. The rafting sequence requires 14.5 GFLOPs for inference, whereas the skiing sequence requires approximately one-third of that computation.

Consistent with Figure~\ref{fig:Distribution}, the rafting video's high frame-to-frame variability and spatial detail drive higher $R^{s}$ and $R^{t}$ levels. Conversely, the skiing video's uniform background and localized motion reduce spatio-temporal complexity, lowering computational demand.

\begin{figure}[h]
    \centering
    \begin{subfigure}{0.41\linewidth}
        \centering
        \includegraphics[width=\linewidth]{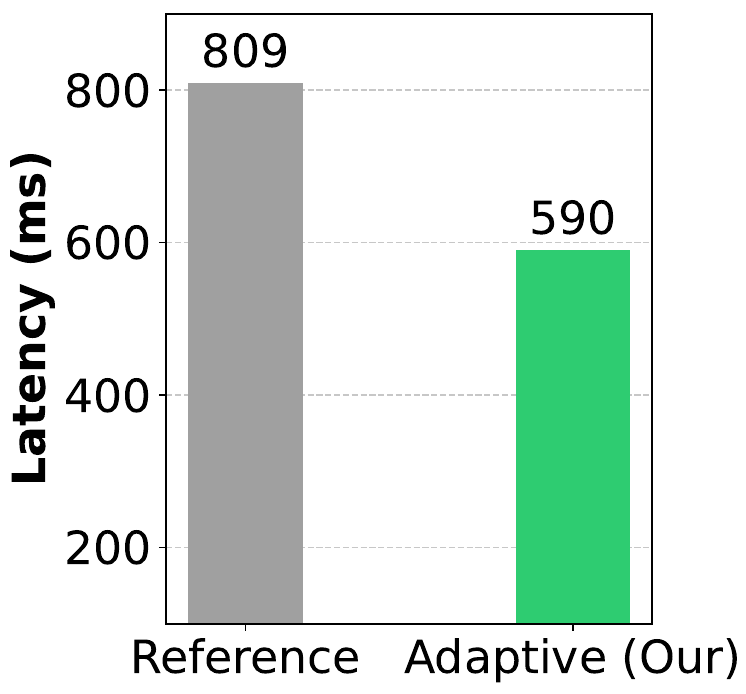}
        \caption{Inference latency}
        \label{fig:latency_jetson}
    \end{subfigure}
    \hfill
    \begin{subfigure}{0.41\linewidth}
        \centering
        \includegraphics[width=\linewidth]{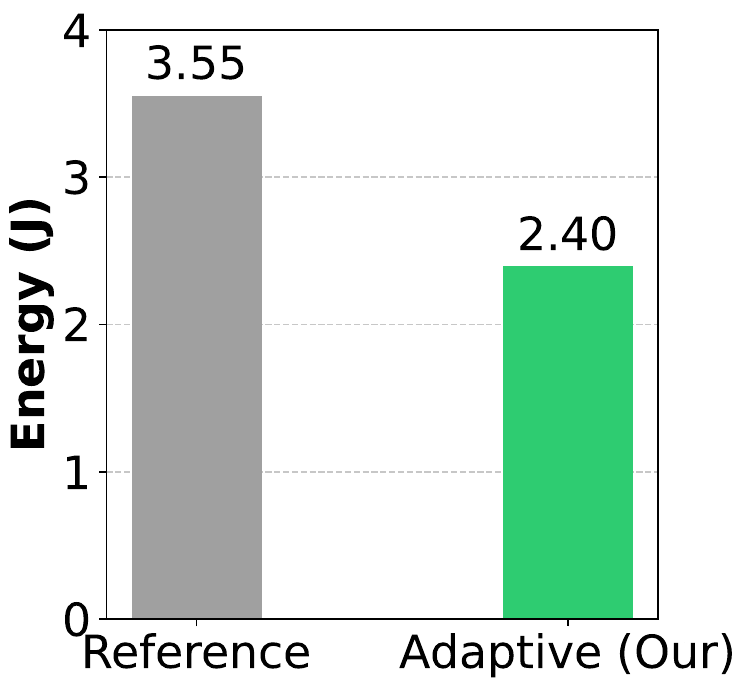}
        \caption{Inference Energy}
        \label{fig:energy_jetson}
    \end{subfigure}
    \caption{R(2+1)D latency and energy on Jetson Nano.}
    \label{fig:jetson_overhead}
    \vspace{-0.15in}
\end{figure}

\begin{figure}
    \centering
    \includegraphics[width=0.9\linewidth]{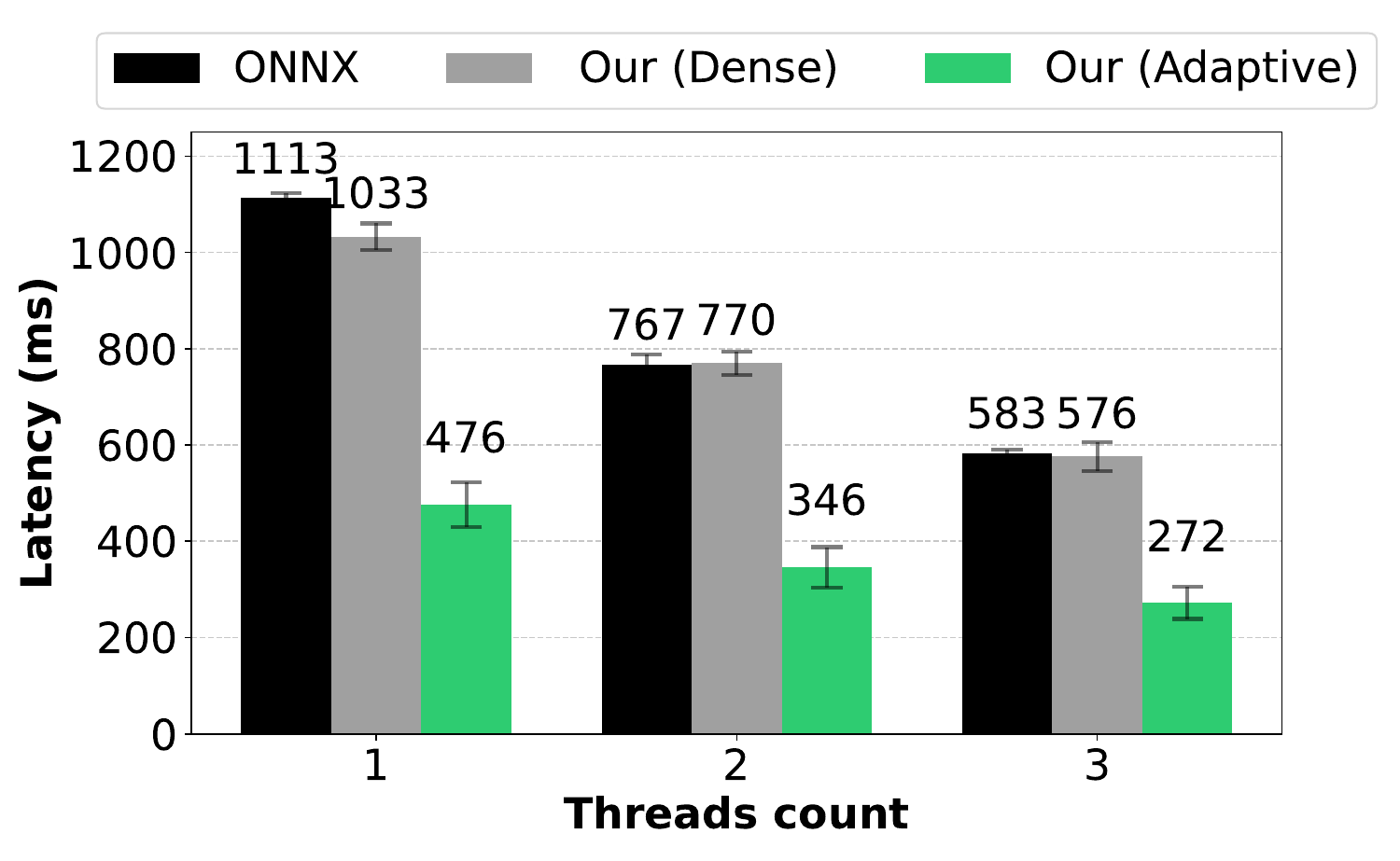}
    \caption{C3D inference latency on CPU (UCF101) vs. baseline.}
    \label{fig:cpu_infer}
    \vspace{-0.2in}
\end{figure}

We evaluate computational efficiency on two edge architectures: an NVIDIA Jetson Nano (4GB GPU) and a Samsung S22 (Qualcomm Snapdragon 8 Gen 1 ARM CPU). To optimize platform-specific constraints—branching penalties on GPUs and vectorization capabilities on CPUs—we employ temporal/channel pruning for the GPU and temporal/spatial pruning for the CPU.GPU Implementation: We implemented R(2+1)D on the Jetson Nano using UCF101. The decomposed convolution structure of R(2+1)D is less memory-intensive, suiting the GPU's restricted RAM. Latency was measured via GPU clock time, and power was recorded using an external PSU (reported as delta over idle). 

As shown in Figures~\ref{fig:latency_jetson} and~\ref{fig:energy_jetson}, our adaptive method achieves a $1.37\times$ latency reduction and $1.47\times$ better energy efficiency over the PyTorch baseline. On the mobile CPU, we developed a custom C3D kernel optimized for ARMv9. We leveraged Neon SIMD instructions and FP16 (Half-Precision) to mitigate memory bottlenecks. Figure~\ref{fig:cpu_infer} compares three implementations: an ONNX baseline (with standard Conv-BN-ReLU fusions), our custom Dense kernel, and our Adaptive Pruning kernel.Across 1, 2, and 3 threads, our Dense kernel matches the ONNX baseline; however, the Adaptive Pruning method delivers a substantial $2.22\times$ average speedup. 
The sub-linear scaling observed with increased thread counts suggests the system is memory-bandwidth bound rather than compute-bound. This indicates that the shared memory bus saturates before the execution units, presenting an opportunity for even more aggressive pruning to further alleviate memory pressure.

\section{Conclusion}
We propose \textit{DANCE}, a fine-grained activation pruning framework for 3D CNNs, enabling energy-efficient inference on edge devices. Using activation variance to guide pruning, the method adapts dynamically to input complexity. It surpasses SOTA pruning approaches in both efficiency and accuracy, even after accounting for the controller overhead. The framework generalizes across datasets and can be extended to other high-cost models such as vision transformers.

\bibliographystyle{ieeenat_fullname}
\bibliography{ashiqur}

\end{document}